\documentclass{article}

\usepackage{microtype}
\usepackage{graphicx}
\usepackage{subfigure}
\usepackage{booktabs} 

\usepackage{times}
\usepackage{graphicx} 
\usepackage{subfigure} 

\usepackage{natbib}

\usepackage{algorithm}
\usepackage{algorithmic}

\usepackage{dsfont}
\usepackage{amsmath}
\usepackage{amsfonts}
\usepackage{amssymb}
\usepackage{amsthm}
\usepackage{subfigure}
\usepackage{epsfig}
\usepackage{color}
\usepackage{enumerate}
\usepackage{placeins}

\usepackage{times}
\usepackage{graphicx} 
\usepackage{subfigure} 

\usepackage{algorithm}
\usepackage{algorithmic}

\usepackage[dvipsnames]{xcolor}

\usepackage{hyperref}

\usepackage{hyperref}

\newtheorem{lemma}{Lemma}
\newtheorem{theorem}{Theorem}

\newtheorem{definition}{Definition}
\newtheorem{assumption}{Assumption}
\newtheorem{remark}{Remark}

\theoremstyle{definition}

\usepackage[accepted]{icml2018}

\icmltitlerunning{Quickshift++: Provably Good Initializations for Sample-Based Mean Shift}

\begin{document}

\twocolumn[
\icmltitle{Quickshift++: Provably Good Initializations for Sample-Based Mean Shift}




\begin{icmlauthorlist}
\icmlauthor{Heinrich Jiang}{google}
\icmlauthor{Jennifer Jang}{uber}
\icmlauthor{Samory Kpotufe}{princeton}
\end{icmlauthorlist}

\icmlaffiliation{google}{Google Research, Mountain View, CA}
\icmlaffiliation{uber}{Uber Inc, San Francisco, CA}
\icmlaffiliation{princeton}{Princeton University, Princeton, NJ}

\icmlcorrespondingauthor{Heinrich Jiang}{heinrich.jiang@gmail.com}

\icmlkeywords{}

\vskip 0.3in
]



\printAffiliationsAndNotice{}  

\begin{abstract} 
We provide initial seedings to the Quick Shift clustering algorithm, which approximate the locally high-density regions of the data. Such seedings act as more {\it stable} and {\it expressive} cluster-cores than the singleton modes found by Quick Shift. We establish statistical consistency guarantees for this modification. We then show strong clustering performance on real datasets as well as promising applications to image segmentation.
\end{abstract} 

\section{Introduction}\label{section:intro}

Quick Shift \cite{vedaldi2008quick} is a mode-seeking based clustering algorithm that has a growing popularity in computer vision.
It proceeds by repeatedly moving each sample to its closest sample point  that has higher empirical density if one exists within a $\tau$-radius ball, otherwise we stop. Thus each path ends at a point which can be  viewed as a local mode of the empirical density. Then, points that end up at the same mode are assigned to the same cluster.
The most popular choice of empirical density function is the Kernel Density Estimator (KDE) with Gaussian Kernel. The algorithm also appears in \citet{rodriguez2014clustering}.

Quick Shift was designed as a faster alternative to the well-known Mean Shift algorithm \cite{cheng1995mean, comaniciu2002mean}.
Mean Shift is equivalent to performing a gradient ascent of the KDE starting at each sample until convergence \cite{arias2016estimation}. Samples that
correspond to the same points of convergence are in the same cluster and the points of convergence are taken to be the estimates of the modes.
Thus, both procedures hill-climb to the local modes of the empirical density function and cluster based on these modes. 
The key differences are that Quick Shift restricts the steps to sample points (and thus is a sample-based version of Mean Shift) and has the extra $\tau$ parameter which allows it to merge close segments together.

One of the drawbacks of these two procedures, as well as many mode-seeking based clustering algorithms, is that the point-modes of the density functions are often poor representations of the clusters. This will happen when the high-density regions within a cluster are of arbitrary shape and have some variations causing the underlying density function to have possibly many apparent, but not so salient modes. In this case, such procedures asymptotically recover all of the modes separately, leading to over-segmentation. To combat this effect, practitioners often increase the kernel bandwidth, which makes the density estimate more smooth. However, this can cause the density estimate to deviate too far from the original density we are intending to cluster based on.\footnote{KDE with Gaussian kernel and bandwidth $h$ approximates the underlying density convolved with a Gaussian with mean $0$ and covariance $h^2{\bf I}$. Thus, the higher $h$ is, the more the KDE deviates from the original density.} Thus, practitioners may not wish to identify the clusters based on the point-modes of the density function, but rather identify them based on {\it locally high density regions} of the dataset (See Figure~\ref{modesvsclustercore}).\footnote{Over-segmentation is also dealt with in Quick Shift via the $\tau$ parameter, but a threshold for the distance between two modes which should be clustered together is hard to determine in practice. Moreover, there may not even be a good setting of $\tau$ which works everywhere in the input space.}

We propose modeling these locally high-density regions as {\it cluster-cores} (to be precisely defined later),
which can be of arbitrary shape, size, and density level, and are thus better suited at capturing the possibly complex topological properties of clusters that can arise in practice. In other words, these cluster-cores are better at expressing the clusters and are more stable as they are less sensitive to the small fluctuations that can arise in the empirical density function.
We parameterize the cluster-core by $\beta$ where $0 < \beta < 1$, which determines how much the density is allowed to vary within the cluster-core. We estimate them from a finite sample using a minor modification of the MCores algorithm of \citet{jiang2017modal}.

We introduce Quickshift++, which first estimates these cluster-cores, and then runs the Quick Shift based hill-climbing procedure on each remaining sample until it reaches a cluster-core. Samples that end up in the same cluster-core are assigned to the same cluster; thus, the cluster-cores can be seen as representing the high-confidence regions within each cluster. We utilize the $k$-NN density estimator as our empirical density.

Despite the simplicity of our approach, we show that Quickshift++ considerably outperforms the popular density-based clustering algorithms, while being efficient. Another desirable property of Quickshift++ is that it is simple to tune its two hyperparameters $\beta$ and $k$.\footnote{The $\tau$ parameter from Quick Shift is unnecessary here because we climb until we reach a cluster-core as our stopping condition.}
 We show that a few settings of $\beta$ turn out to work for a wide range of applications and that the procedure is stable in choices of $k$.
 
We then give a novel statistical consistency analysis for Quickshift++ which provides guarantees that points within a cluster-core's attraction regions (to be described later) are correctly assigned. 
We also show promising results on image segmentation, which further validates the desirability of using cluster-cores on real-data applications. 

\vspace{-0.4cm}
\begin{figure}[H]
\label{modesvsclustercore}
\begin{center}
\includegraphics[width=0.23\textwidth]{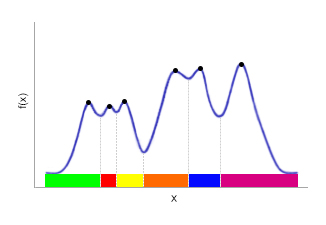}
\includegraphics[width=0.23\textwidth]{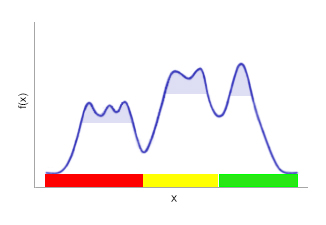}
\end{center}

\caption{
It can often be the case that the locally high-density regions are of arbitrary shape and fluctuations within them lead to many apparent modes. {\bf Left}: Mode-seeking clustering procedures often lead to over-segmentation. {\bf Right}: It may be more desirable to use cluster-cores (shaded), which allows fluctuations within arbitrarily-shaped regions of locally high density.}
\end{figure}


\section{Related Works and Contributions}

We show that Quickshift++ is a new and powerful addition to the family of clustering procedures known as {\it density-based} clustering, which most notably includes DBSCAN \cite{ester1996density} and Mean Shift \cite{cheng1995mean}.
Such procedures operate on the estimated density function based on a finite sample to recover structures in the density function that
ultimately correspond to the clusters. There are several advantages of density-based clustering over classical objective-based procedures such as k-means and spectral clustering. Density-based procedures can {\it automatically} detect the number of clusters, while objective-based procedures typically require this as an input. Density-based clustering algorithms also make little assumptions on the shapes of the clusters as well as their relative positions.

Density-based clustering procedures can roughly be classified into two categories: hill-climbing based approaches (discussed earlier, which includes both Mean Shift and Quick Shift) and density-level set based approaches. We now discuss the latter approach, which takes the connected components of the density-level set defined by $\{ x : f(x) \ge \lambda\}$ for some density level $\lambda$ as the clusters. This statistical notion of clustering traces back to \citet{hartigan1975clustering}. 
Since then, there has been extensive work done, e.g. \citet{tsybakov1997nonparametric,cadre2006kernel,rigollet2009optimal,singh2009adaptive,chaudhuri2010rates,rinaldo2010generalized,kpotufe2011pruning,balakrishnan2013cluster,chaudhuri2014consistent,chen2017density}. More recently, \citet{sriperumbudur2012consistency,jiang2017density} show that the popular DBSCAN algorithm turns out to converge to these clusters. However, one of the main drawbacks of this approach is that the density-level $\lambda$ is fixed and thus such methods perform poorly when the clusters are at different density-levels. Moreover, the question of how to choose $\lambda$ remains (e.g. \citet{steinwart2011adaptive}).

\citet{jiang2017modal} provide an alternative notion of clusters, called modal-sets, which are regions of flat density which are local maximas of the density. They can be of arbitrary shape, dimension, or density.
They provide a procedure, MCores, which estimates these with consistency guarantees.
Our notion of cluster-core is similar to modal-sets, but the density within a cluster-core is allowed to vary by a substantial amount in order to capture such variations seen in real data as a the flat density of modal-sets may be too restrictive in practice. It turns out that a small modification of MCores allows us to estimate these cluster-cores. Thus Quickshift++ has the advantage over DBSCAN in that clusters can be at any density level and that furthermore, the density levels are chosen adaptively.

Mcores however consists of an over-simplistic final clustering: it simply assigns each point to its closest modal-set, while in practice, clusters tend not to follow the geometry induced by the Euclidean metric. Quickshift++ on the other hand clusters the remaining points by a hill-climbing method which we show is far better in practice.

Thus, Quickshift++ combines the strengths of both density-based clustering approaches while avoiding many of their weaknesses. In addition to the general advantages of density-based clustering algorithms shared by both approaches, it is also able to both (1) recover clusters at varying density levels and
(2) not suffer from the over-segmentation issue described in Figure~\ref{modesvsclustercore}. To our knowledge, no other procedure has been shown to have this property.

For our theoretical analysis, we give guarantees about Quickshift++'s ability to recover the clusters based on attraction regions defined by the gradient flows. \citet{wasserman2014feature,arias2016estimation} showed that Mean Shift's iterates approximate the gradient flows. Some progress has been made in understanding Quick Shift \cite{jiang2017on,verdinelli2018analysis}. There are also related lines of work in mode clustering e.g. \citep{li2007nonparametric,chacon2012clusters,genovese2016non,chen2016comprehensive}. In this paper, we show that Quickshift++ recovers the interior of its attraction region, thus adding to our statistical understanding of hill-climbing based clustering procedures.


\section{Algorithm}\label{section:algorithm}
\subsection{Basic Definitions}

Let $X_{[n]} = \{x_1,...,x_n\}$ be $n$ i.i.d. samples drawn from an unknown density $f$, defined over the Lebesgue measure on $\mathbb{R}^d$. Suppose that $f$ has compact support $\mathcal{X}$.

Our procedure will operate on the $k$-NN density estimator:

\begin{definition} \label{kNNdensity} Let $r_k(x) := \inf \{ r > 0 : |B(x, r) \cap X_{[n]}| \ge k \}$, i.e., the distance from $x$ to its $k$-th nearest neighbor. Define the $k$-NN density estimator as
\begin{align*}
f_k(x) := \frac{k}{n\cdot v_d\cdot r_k(x)^d},
\end{align*}
where $v_d$ is the volume of a unit ball in $\mathbb{R}^d$.
\end{definition}

\subsection{Cluster-Cores}

We define the cluster core with respect to fixed fluctuation parameter $\beta$ as follows.
 \begin{definition}\label{def:clustercore} 
 Let $0 < \beta < 1$. 
Closed and connected set $M \subset \mathcal{X}$ is a cluster-core if $M$ is a connected component (CC) of $\{ x \in \mathcal{X} : f(x) \ge (1 - \beta)\cdot \max_{x' \in M} f(x') \}$.
\end{definition}
Note that when $\beta \rightarrow 0$, then the cluster-cores become the modes or local-maximas of $f$. When $\beta \rightarrow 1$, then the cluster-core becomes the entire support $\mathcal{X}$. We next give a very basic fact about cluster-cores, that they do not overlap.
\begin{lemma}
Suppose that $M_1$, $M_2$ are distinct cluster-cores of $f$. Then $M_1 \cap M_2 = \emptyset$.
\end{lemma}
\begin{proof}
Suppose otherwise. We have that $M_1$ and $M_2$ are CCs of $\{ x \in \mathcal{X} : f(x) \ge \lambda_1\}$ and $\{ x \in \mathcal{X} : f(x) \ge \lambda_2\}$, respectively for some $\lambda_1, \lambda_2$. Clearly, if $\lambda_1 = \lambda_2$, then it follows that $M_1 = M_2$. Then, without loss of generality, let $\lambda_1 < \lambda_2$. Then since the CCs of $\{ x \in \mathcal{X} : f(x) \ge \lambda_2\}$ are nested in the CCs of $\{ x \in \mathcal{X} : f(x) \ge \lambda_1\}$, then it follows that $M_2 \subseteq M_1$. Then, $\lambda_2 = (1 - \beta) \sup_{x \in M_2} f(x) \le (1 - \beta) \sup_{x \in M_1} f(x) = \lambda_1$, a contradiction. As desired.
\end{proof}

Algorithm~\ref{alg:clustercore} is a simple modification of MCores by \citet{jiang2017modal}. The difference is that we use a multiplicative fluctuation parameter $\beta$, while \citet{jiang2017modal} uses an additive one. The latter requires knowledge of the scale of the density function, which is difficult to determine in practice. Moreover, the multiplicative fluctuation adapts to clusters at varying density levels more reasonably than a fixed additive fluctuation. It uses the levels of the mutual $k$-NN graph of the sample points, defined below.
\begin{definition}
 Let $G(\lambda)$ denote the $\lambda$-level of the mutual $k$-NN graph with vertices $\{ x \in X_{[n]} : f_k(x) \ge \lambda\}$ and an edge between $x$ and $x'$ iff $||x - x'|| \le \min \{r_k(x), r_k(x') \}$. 
 \end{definition}
 It is known that $G(\lambda)$ approximates the CCs of the $\lambda$-level sets of the true density, defined as $\{ x \in \mathcal{X} : f(x) \ge \lambda\}$ see e.g. \cite{chaudhuri2010rates}. Moreover, it can be seen that the CCs of $G(\lambda)$ forms a hierarchical nesting structure as $\lambda$ decreases. 

 Algorithm~\ref{alg:clustercore} proceeds by performing a top-down sweep of the levels of the mutual $k$-NN graph, $G(\lambda)$. As $\lambda$ decreases, it is clear that more nodes appear and that connectivity increases. In other words, as we scan top-down, the CCs of $G(\lambda)$ become larger, some CCs can merge, or new CCs can appear. When a new CC appears at level $\lambda$, then intuitively, it should correspond to a local maxima of $f$, which appears at a density level approximately $\lambda$. This follows from the fact that the CCs of $G(\lambda)$ approximates the CCs of $\{x \in \mathcal{X} : f(x) \ge \lambda\}$. Thus, the idea is that when a new CC appears in $G(\lambda)$, then we can take the corresponding CC in $G(\lambda - \beta \lambda)$ (which is the density level $(1-\beta)$ times that of the highest point in the CC) to estimate the cluster-core. 

\begin{algorithm}[H]
   \caption{MCores (estimating cluster-cores)}
   \label{alg:clustercore}
\begin{algorithmic}
   \STATE Parameters $k$, $\beta$
   \STATE Initialize $\widehat{\mathcal{M}}:= \emptyset$.
   \STATE Sort the $x_i$'s in decreasing order of $f_k$ values (i.e. $f_k(x_i) \geq f_k(x_{i+1})$). 
   \FOR{$i=1$ {\bfseries to} $n$}
   \STATE Define $\lambda := f_k(x_i)$.
   \STATE Let $A$ be the CC of $G(\lambda - \beta \lambda)$ containing $x_i$. 
   \IF{$A$ is disjoint from all cluster-cores in $\widehat{\mathcal{M}}$}
   \STATE Add $A$ 
   to $\widehat{\mathcal{M}}$. 
    \ENDIF
   \ENDFOR
   \STATE \textbf{return} $\widehat{\mathcal{M}}$.
\end{algorithmic}
\end{algorithm}

\vspace{-0.5cm}
\begin{algorithm}[H]
   \caption{Quickshift++}
   \label{alg:quickshiftpp}
\begin{algorithmic}
   \STATE Let $\widehat{\mathcal{M}}$ be the cluster-cores obtained by running Algorithm~\ref{alg:clustercore}.
   \STATE Initialize directed graph $G$ with vertices $\{x_1,...,x_n\}$ and no edges.
   \FOR{$i=1$ {\bfseries to} $n$}
   	\STATE If $x_i$ is not in any cluster-core, then add to $G$ an edge from $x_i$ to its closest sample $x \in X_{[n]}$ such that $f_k(x) > f_k(x_i)$.
   \ENDFOR
   \STATE For each cluster-core $M \in \widehat{\mathcal{M}}$, let $\widehat{\mathcal{C}}_M$ be the points $x \in X_{[n]}$ such that the directed path in $G$ starting at $x$ ends in $M$.
   \STATE \textbf{return} $\{\widehat{\mathcal{C}}_M : M \in  \widehat{\mathcal{M}}\}$.
\end{algorithmic}
\end{algorithm}

\subsection{Quickshift++}

Quickshift++ (Algorithm~\ref{alg:quickshiftpp}) proceeds by first running Algorithm~\ref{alg:clustercore} to obtain the cluster-cores, and then moving each sample point to its nearest neighbor that has higher $k$-NN density until it reaches some cluster-core. All samples that end up in the same cluster-core after the hill-climbing are assigned to the same cluster. Note that since the highest empirical density sample point is contained in a cluster-core, it follows that each sample point not in a cluster-core will eventually be assigned to a unique cluster-core. Thus, Quickshift++ provides a clustering assignment of {\it every} sample point.


\begin{remark}
Although it seems a similar procedure could have been constructed by using Mean Shift in place of Quick Shift, Mean Shift could have convergence outside of the estimated cluster-cores, while Quick Shift guarantees that each sample outside of a cluster-core get assigned to some cluster-core.
\end{remark}

\subsection{Implementation}
The implementation details for the MCores modification can be inferred from \citet{jiang2017modal}.
This step runs in $O(nk\cdot \alpha(n))$ where $\alpha$ is the Inverse Ackermann function \cite{cormen2009introduction}, in addition to the time it takes to compute the $k$-NN sets for the $n$ sample points. To cluster the remaining points, for each sample not in a cluster-core, we must find its nearest sample of higher $k$-NN density. Although this is worst-case $O(n)$ time for each sample point, fortunately we see that in practice (as long as $k$ is not too small) for the vast majority of cases, the nearest sample with higher density is within the $k$-nearest neighbor set so it only takes $O(k)$ in most cases. It is an open problem whether there the nearest sample with higher empirical density is often in its $k$-NN set. Code release is at \url{https://github.com/
google/quickshift}.

\section{Theoretical Analysis}\label{section:theory}
For the theoretical analysis, we make first the following regularity assumption, that the density is continuously differentiable and lower bounded on $\mathcal{X}$.

\begin{assumption}\label{assumption:diff}
$f$ is continuously differentiable on $\mathcal{X}$ and there exists $\lambda_0 > 0$ such that $\inf_{x \in \mathcal{X}} f(x) \ge \lambda_0$.
\end{assumption}

Let $M_1,...,M_C$ be the cluster-cores of $f$. Then we can define the following notion of attraction region for each cluster-core based on the gradient ascent curve or flow. This is similar to notions of attraction regions for some previous analyses of mode-based clustering, such as \citet{wasserman2014feature, arias2016estimation}, where the intuition is that attraction regions are defined based by following the direction of the gradient of the underlying density. In our situation, instead of an attraction region defined as all points which flow towards a particular point-mode, the attraction region is defined around a cluster-core.

\begin{definition} [Attraction Regions]
Let path $\pi_x : \mathbb{R} \rightarrow \mathbb{R}^d$ satisfy $\pi_x(0) = x$, $\pi_x'(t) = \nabla f(\pi_x(t))$. 
For cluster-core $M_i$, its attraction region $\mathcal{A}_i$ is the set of points $x \in \mathcal{X}$ that satisfy $\lim_{t\rightarrow \infty}\pi_x(t) \in M_i$.
\end{definition}
It is clear that these attraction regions are well-defined. The flow path is well-defined since the density is differentiable and since each cluster-core is defined as a CC of a level set, the density must decay around its boundaries. In other words, once an ascent path reaches a cluster-core, it cannot leave the cluster-core. 

However, it is in general not the case that the space can be partitioned into attraction regions. For example, if a flow reaches a saddle point, it will get stuck there and thus any point whose flow ends up at a saddle point will not belong to any attraction region. In this paper, we only give guarantees about the clustering of points which are in an attraction region. 



The next regularity assumption we make is that the cluster-cores are on the interior of the attraction region (to avoid situations such as when the cluster-cores intersect with the boundary of the input space).
\begin{assumption}\label{assumption:interior}
There exists $R_0 > 0$ such that $M_i + B(0, R_0) \subseteq \mathcal{A}_i$ for $i = 1,...,C$, where $M + B(0, r)$ denotes $\{ x : \inf_{y \in M} ||x - y|| \le r\}$.
\end{assumption}

\begin{definition} [Level Set]
The $\lambda$ level set of $f$ is defined as $L_f(\lambda) := \{x \in \mathcal{X}: f(x) \ge \lambda \}$.
\end{definition}

The next assumption says that the level sets are continuous w.r.t. the level in the following sense where we denote the $\epsilon$-interior of $A$ as $A^{\ominus \epsilon} := \{x \in A, \inf_{y \in \partial A} ||x- y|| \ge \epsilon \}$ ($\partial A$ is the boundary of $A$):

\begin{assumption} [Uniform Continuity of Level Sets]\label{assumption:levelset_continuity}
For each $\epsilon > 0$, there exists $\delta > 0$ such that for $0 < \lambda \le \lambda' \le ||f||_\infty$ with $|\lambda - \lambda'| < \delta$, then
$L_f(\lambda)^{\ominus \epsilon} \subseteq L_f(\lambda')$.
\end{assumption}
This ensures that there are no approximately flat areas in which the procedure may get stuck at. The assumption is borrowed from \cite{jiang2017on}. Finally, we need the following regularity condition which ensures that level sets away from cluster-cores do not get arbitrarily thin. This is adapted from standard analyses of level-set estimation (e.g. Assumption B of \citet{singh2009adaptive}).

\begin{assumption}\label{assumption:levelset_reg}
Let $\mu$ denote the Lebesgue measure on $\mathbb{R}^d$.
For any $r > 0$, there exists $\sigma > 0$ such that the following holds for any connected component $A$ of any level-set of $f$ which is not contained in $M_i$ for any $i$: $\mu(B(x, r) \cap A) \ge \sigma$ for all $x \in A$.
\end{assumption}

For our consistency results, we prove that Quickshift++ can cluster the sample points in the $(R, \rho)$-interior of an attraction region (defined below) for each cluster-core properly where $R, \rho > 0$ are fixed and can be chosen arbitrarily small.

\begin{definition} [$(R, \rho)$-interior of Attraction Regions]
Define the $(R, \rho)$-interior of $\mathcal{A}_i$, denoted as $\mathcal{A}^{(R, \rho)}_i$, as the set of points $x_0 \in \mathcal{A}_i$ such that
each path $\mathcal{P}$ from $x_0$ to any point $y \in \partial \mathcal{A}_i$ satisfies the following.
\begin{align*}
\sup_{x \in \mathcal{P}} \inf_{x' \in B(x, R)} f(x') \ge \sup_{x' \in B(y, R)} f(x') + \rho.
\end{align*}
\end{definition}
In other words, points in the interior satisfy the property that any path leaving its attraction region must sufficiently decrease in density at some point. This decrease threshold is parameterized by $R$ and $\rho$.
\begin{figure}[h]
\begin{center}
\includegraphics[width=0.3\textwidth]{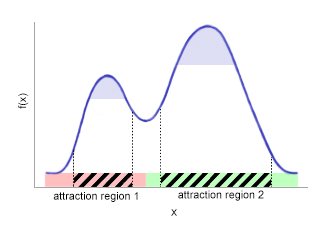}
\end{center}
\vspace{-0.5cm}
\caption{Illustration of interior of attraction region in 1-dimension. The pink and green shaded regions indicated the two attraction regions. The striped parts show the corresponding interiors of the attraction regions.}
\end{figure}

\vspace{-0.3cm}

We first give a guarantee on the first step of MCores recovers, that the cluster-cores are reasonably recovered. The proof follows from the analysis of \citet{jiang2017modal} by replacing modal-sets with cluster-cores, and the results match up to constant factors. The proof is omitted here.

\begin{theorem}\label{theo:mcores}[Adapted from Theorem 3, 4 of \citet{jiang2017modal}]
Suppose that Assumptions~\ref{assumption:diff},~\ref{assumption:levelset_continuity}, and~\ref{assumption:levelset_reg} hold. Let $0 < \beta < 1$, $\epsilon, \delta > 0$ and suppose that $k \equiv k(n)$ is chosen such that $\log^2 n / k \rightarrow 0$ and $n^{4/(4+d)} / k \rightarrow 0$. Let $M_1,...,M_C$ be the cluster-cores of $f$.
Then for $n$ sufficiently large depending on $f$, $\delta$, $\epsilon$, and $\beta$, with probability at least $1 - \delta$, MCores returns $C$ cluster-core estimates $\widehat{M_1},...,\widehat{M_C}$ such that $M_i \cap X_{[n]} \subseteq \widehat{M_i} \subseteq M_i + B(0, \epsilon)$ for $i \in 1,...,C$.
\end{theorem}

\begin{remark}
The original result from \citet{jiang2017modal} is about $\epsilon$-approximate modal-set which are defined as level-sets whose density has range $\epsilon$. 
Our notion of cluster-core is similar, but the range is a $\beta$-proportion of the highest density level within the level-set. Using a proportion is more interpretable and thus more useful, as the scale of the density function is difficult to determine in practice.
\end{remark}

In other words, with high probability, MCores estimates each cluster-core bijectively and that for each cluster-core, MCores' estimate contains all of the sample points and that the estimate does not over-estimate by much.

We now state the main result, which says that as long as the cluster-cores are sufficiently well estimated (up to a certain Hausdorff error) by MCores (via previous theorem), then Quickshift++ will correctly cluster the $(R, \rho)$-interiors of the attraction regions with high probability.

\begin{theorem}\label{theo:main}
Suppose that Assumptions~\ref{assumption:diff},~\ref{assumption:interior},~\ref{assumption:levelset_continuity}, and~\ref{assumption:levelset_reg} hold.
Let $0 < R < R_0$ and $\rho, \delta> 0$. Suppose that $k \equiv k(n)$ is chosen such that $\log^2 n / k \rightarrow 0$ and $n^{4/(4+d)} / k \rightarrow 0$.
Suppose that $\widehat{M}_1,...,\widehat{M}_C$ are the cluster-cores returned by Algorithm~\ref{alg:clustercore} and satisfy 
$M_i \cap X_{[n]} \subseteq \widehat{M}_i \subseteq M_i + B(0,R/4)$ for $i =1,...,C$.
Then for $n$ sufficiently large depending on $f$, $\rho$, $\delta$ and $R$, the following holds with probably at least $1 - 2\delta$ uniformly in $x \in \mathcal{A}_{i}^{(R,\rho)} \cap X_{[n]}$ and $i \in [C]$:
Quickshift++ clusters $x$ to the cluster corresponding to $M_i$.
\end{theorem}

\subsection{Proof of Theorem~\ref{theo:main}}
We require the following uniform bound on $k$-NN density estimator, which follows from \citet{dasgupta2014optimal}.

\begin{lemma}\label{lemma:knn}
Let $\delta > 0$.
Suppose that $f$ is Lipschitz continuous with compact support $\mathcal{X}$ (e.g. there exists $L$ such that $|f(x) - f(x')| \le L|x-x'|$ for all $x,x' \in \mathcal{X}$)
 and $f$ satsifies Assumption~\ref{assumption:diff}.
Then exists constant $C$ depending on $f$ such that the following holds if $n \ge C_{\delta, n}^2$ with probability at least $1 - \delta$.
\begin{align*}
\sup_{x \in \mathcal{X}} |f_k(x) - f(x)| \le C\left( \frac{C_{\delta, n}}{\sqrt{k}} + \left( \frac{k}{n}\right)^{1/d} \right).
\end{align*}
where $C_{\delta, n} := 16 \log(2/\delta) \sqrt{d \log n}$.
\end{lemma}

We next need the following uniform concentration bound on balls intersected with level-sets, which says that if such a set has large enough probability mass, then it will contain a sample point with high probability.
\begin{lemma}\label{lemma:concentration}
Let $\mathcal{E} := \{B(x, r) \cap L_f(\lambda) : x \in \mathbb{R}^d, r > 0, \lambda > 0\}$. Then the following holds with probability at least $1 - \delta$ uniformly for all $E \in \mathcal{E}$ 
\begin{align*}
    \mathcal{F}(E) \ge C_{\delta, n}\frac{\sqrt{d\log n}}{n} \Rightarrow E \cap X_{n} \neq \emptyset.
\end{align*}
\end{lemma}

\begin{proof}
The indicator functions $1[B(x, f) \cap L_f(\lambda)]$ for $x \in \mathbb{R}^d$, $\lambda > 0$ have VC-dimension $d+1$. This is because the balls over $\mathbb{R}^d$ have VC-dimension $d+1$ and the level-sets $L_f(\lambda)$ has VC-dimension $1$ and thus their intersection has VC-dimension $d+1$ \cite{van2009note}. The result follows by applying Theorem 15 of \citet{chaudhuri2010rates}.
\end{proof}

\begin{proof}[Proof of Theorem~\ref{theo:main}]
Suppose that $x_0 \in \mathcal{A}_i^{(R,\rho)} \cap X_{[n]}$ and Quickshift++ gives directed path $x_0 \rightarrow x_1 \rightarrow \cdots \rightarrow x_L$ where $x_1,...,x_{L-1}$ are outside of cluster-cores and $x_L$ is in a cluster-core but $x_L \not\in \mathcal{A}_{i}$.

We first show that $||x_i - x_{i+1}|| \le R/2$ for $i=0,...,L-1$. By Assumption~\ref{assumption:levelset_continuity} and~\ref{assumption:levelset_reg}, we have that there exists $\tau > 0$ and $\sigma > 0$ such that the following holds uniformly for $i=0,...,L-1$:
{\small
\begin{align*}
    \mu\bigg(B(x_i, R/2) \cap L_f(f(x_i) + \tau)\bigg) \ge \sigma.
\end{align*}
}%
Hence, since the density is uniformly lower bounded by $\lambda_0$, we have
{\small
\begin{align*}
    \mathcal{F}\bigg(B(x_i, R/2) \cap L_f(f(x_i) + \tau)\bigg) \ge \sigma \lambda_0 .
\end{align*}
}%
Then by Lemma~\ref{lemma:concentration}, for $n$ suffiicently large such that $\sigma \lambda_0 > C_{\delta, n}\frac{\sqrt{d\log n}}{n}$, then with probability at least $1 - \delta$ there exists sample point $x_i'$ in $B(x_i, R/2) \cap L_f(f(x_i) + \tau)$ for $i = 0,...,L-1$.

Next, choose $n$ sufficiently large such that by Lemma~\ref{lemma:knn}, we have with probability at least $1 - \delta$ that
\begin{align*}
    \sup_{x \in \mathcal{X}} |f_k(x) - f(x)| \le \min\{\tau, \rho\} / 3.
\end{align*}
Thus, we have 
\begin{align*}
    f_k(x_i') &\ge f(x_i') - \tau/3 \ge f(x_i) + 2\tau/3 \\
    &\ge f_k(x_i) + \tau/3 > f_k(x_i).
\end{align*}
Moreover $||x_i - x_i'|| \le R/2$ and $x_i' \in X_{[n]}$, it follows that 
$||x_i - x_{i+1}||\le R/2$ for $i =0,...,L-1$.

Let $\pi : [0,1] \rightarrow \mathbb{R}^d$ be the piecewise linear path defined by $\pi(j/L) = x_j$ for $j = 0,...,L$. Let $t_2 = \min \{t \in [0,1] : \pi(t) \in \partial \mathcal{A}_i \}$. Then, by definition of $\mathcal{A}_{i}^{(R,\rho)}$, there exists $0 \le t_1 < t_2$ such that $x := \pi(t_1)$ and $y := \pi(t_2)$ satisfies $y \in \partial \mathcal{A}_i$ and
\begin{align*}
\inf_{x' \in B(x, R)} f(x') \ge \sup_{x' \in B(y, R)} f(x') + \rho.
\end{align*}

Thus, there exists indices $p, q \in \{0,...,L-1\}$ such that $p \le q$,
$|x_p - x| \le R$, and $|x_q - y| \le R$. Thus, we have
$f(x_p) \ge f(x_q) + \rho$, but $f_k(x_p) \le f_k(x_q)$. However, we have
\begin{align*}
    f_k(x_p) &\ge f(x_p) - \rho/3 \ge f(x_q) + 2\rho/3\\
    &\ge f_k(x_q) + \rho/3 > f_k(x_q),
\end{align*}
a contradiction, as desired.
\end{proof}

\section{Simulations}

\begin{figure}[!htb]
\begin{center}
\includegraphics[width=0.4\textwidth]{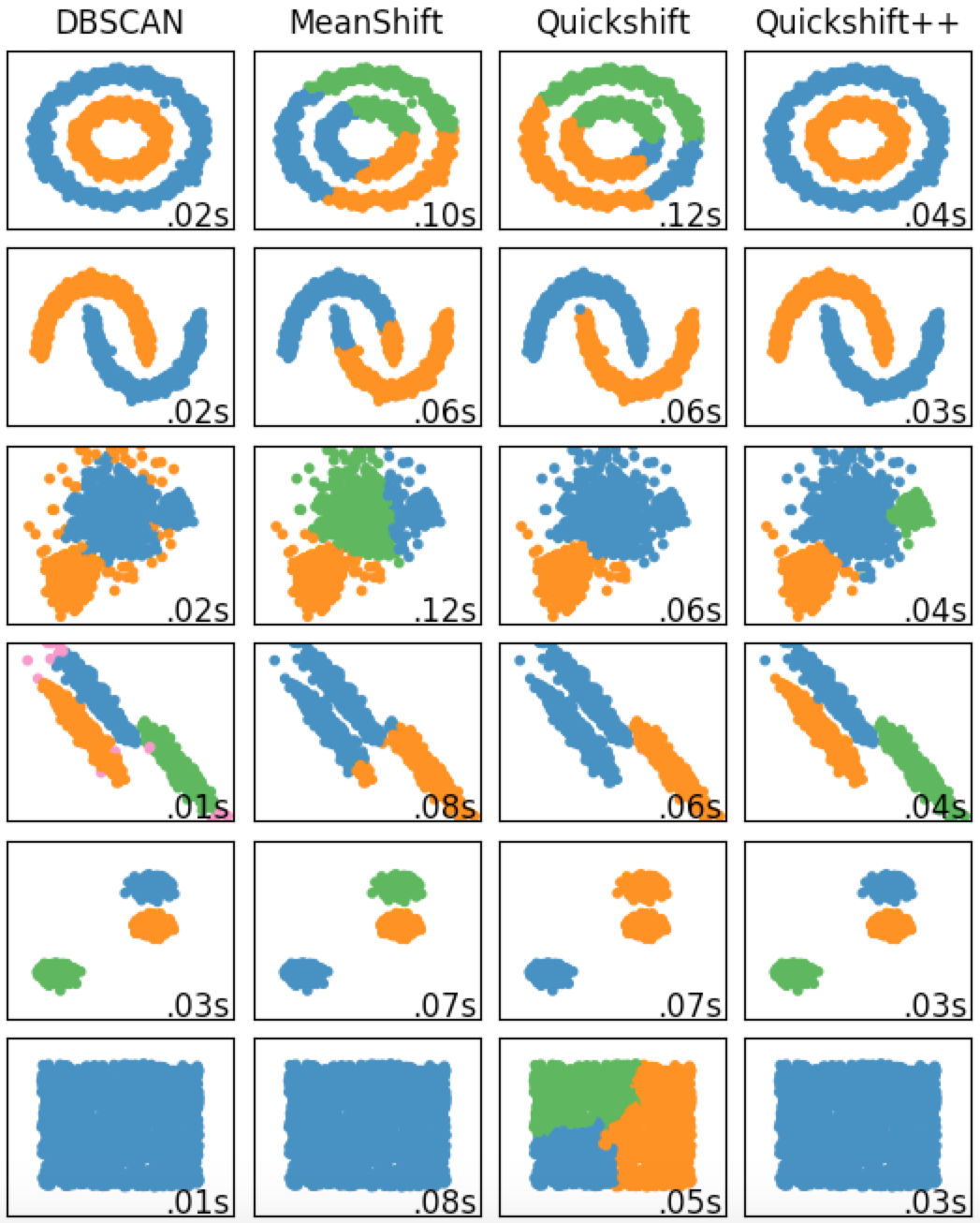}
\end{center}
\caption{\label{fig:simulation} Comparison against other clustering algorithms on toy datasets, adapted from scikit-learn cluster demo. Quickshift++ settings were fixed at $k = 20, \beta = 0.7$ for all the datasets, while the other algorithms were tuned to obtain a reasonable number of clusters.}
\end{figure}

Figure~\ref{fig:simulation} provides simple verification that Quickshift++ provides reasonable clusterings in a wide variety of situations where other density-based procedures are known to fail. For instance, in the two rings dataset (first row), we see that Mean Shift and Quick Shift suffer from the over-segmentation issue coupled with the oversized bandwidth which causes them to recover clusters that have points from both the rings even though the rings are separated. In the three Gaussians dataset (third row), we see that DBSCAN fails because the three clusters are of different density levels and thus no matter which density-level we set, DBSCAN will not be able to recover the three clusters.

\section{Image Segmentation}

In order to apply clustering to image segmentation, we use the following standard approach (see e.g. \citet{felzenszwalb2004efficient}):
 we transform each pixel into a $5$-dimensional vector where two coordinates correspond to the location of the pixel and three correspond to each of the RGB color channels. Then segmentation is done by clustering this $5$-dimensional dataset. 
 
We observed that for Quickshift++, setting $\beta = 0.9$ is reasonable across a wide range of images, $\beta$ was fixed to this value for segmentation here. We compare Quickshift++ to Quick Shift, as the latter is often used for segmentation. Quick Shift often over-segments in some areas and under-segments in other areas under any hyperparameter setting and we showed the settings which provided a reasonable trade-off.  On the other hand Quickshift++ gives us reasonable segmentations in many cases and can capture segments that may be problematic for other procedures.  

As shown in the figures, it moreover has the interesting property of being able to recover segments of widely varying shapes and sizes in the same image, which suggests that modelling the dense regions of the segments as cluster-cores instead of point-modes may be useful as we compare to Quick Shift. Although this is only qualitative, it further suggests that Quickshift++ is a versatile algorithm and begins to show its potential application in many more areas.

\begin{figure}[!htb]
\begin{center}
\includegraphics[width=0.22\textwidth]{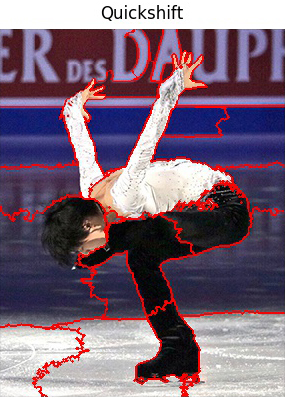}
\includegraphics[width=0.22\textwidth]{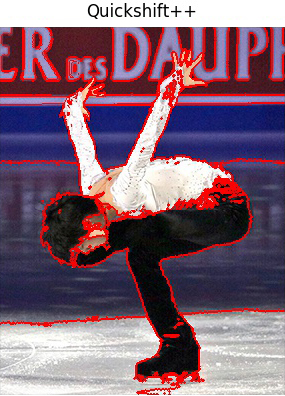}
\end{center}
\caption{Figure skater Yuzuru Hanyu performs at the 2018 Winter Olympics. Quick Shift was set with bandwidth $10$ and Quickshift++ was set with $k = 300$ and $\beta = 0.9$. We see that when compared to Quick Shift, Quickshift++ is able to recover the variations in the background more accurately, including correctly segmenting most of the letters on the wall, while still recovering the structure of Hanyu's costume accurately.}
\end{figure}

\begin{figure}[!htb]
\begin{center}
\includegraphics[width=0.22\textwidth]{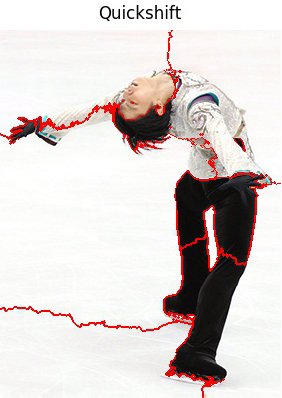}
\includegraphics[width=0.22\textwidth]{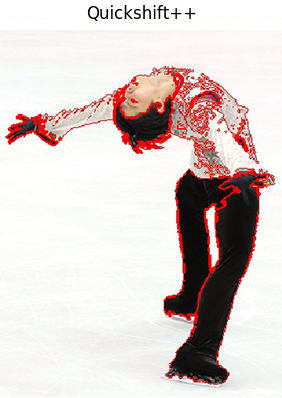}
\end{center}
\caption{Yuzuru Hanyu at the 2017 Rostelecom Cup. Quick Shift was set with bandwidth $15$ and Quickshift++ was set with $k = 50$ and $\beta = 0.9$. Quickshift++ can recover the homogeneous background as a whole, and reasonably separates Hanyu's light-colored costume from the background.}
\end{figure}

\begin{figure}[!htb]
\begin{center}
\includegraphics[width=0.23\textwidth]{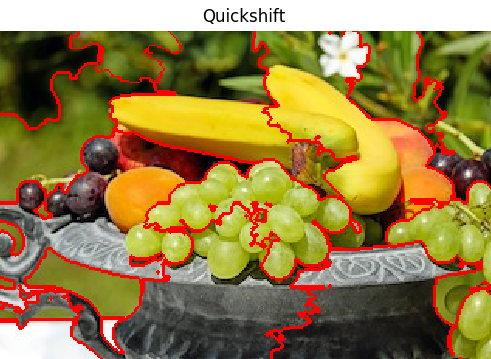}
\includegraphics[width=0.23\textwidth]{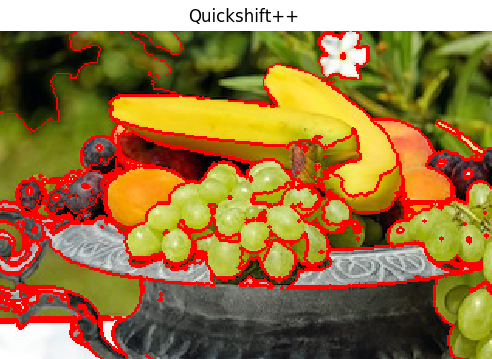}
\end{center}
\caption{Assorted fruit in a metal bowl. For Quick Shift, bandwidth was set to $8$ and for Quickshift++, $k = 100$ and $\beta = 0.9$. Quickshift++ is able to segment most of the fruits in the bowl, while recovering the details of the bowl as well as the structures in the background.}
\end{figure}

\section{Clustering Experiments}
\begin{figure*}[!htb]

\begin{center}
\includegraphics[width=0.75\textwidth]{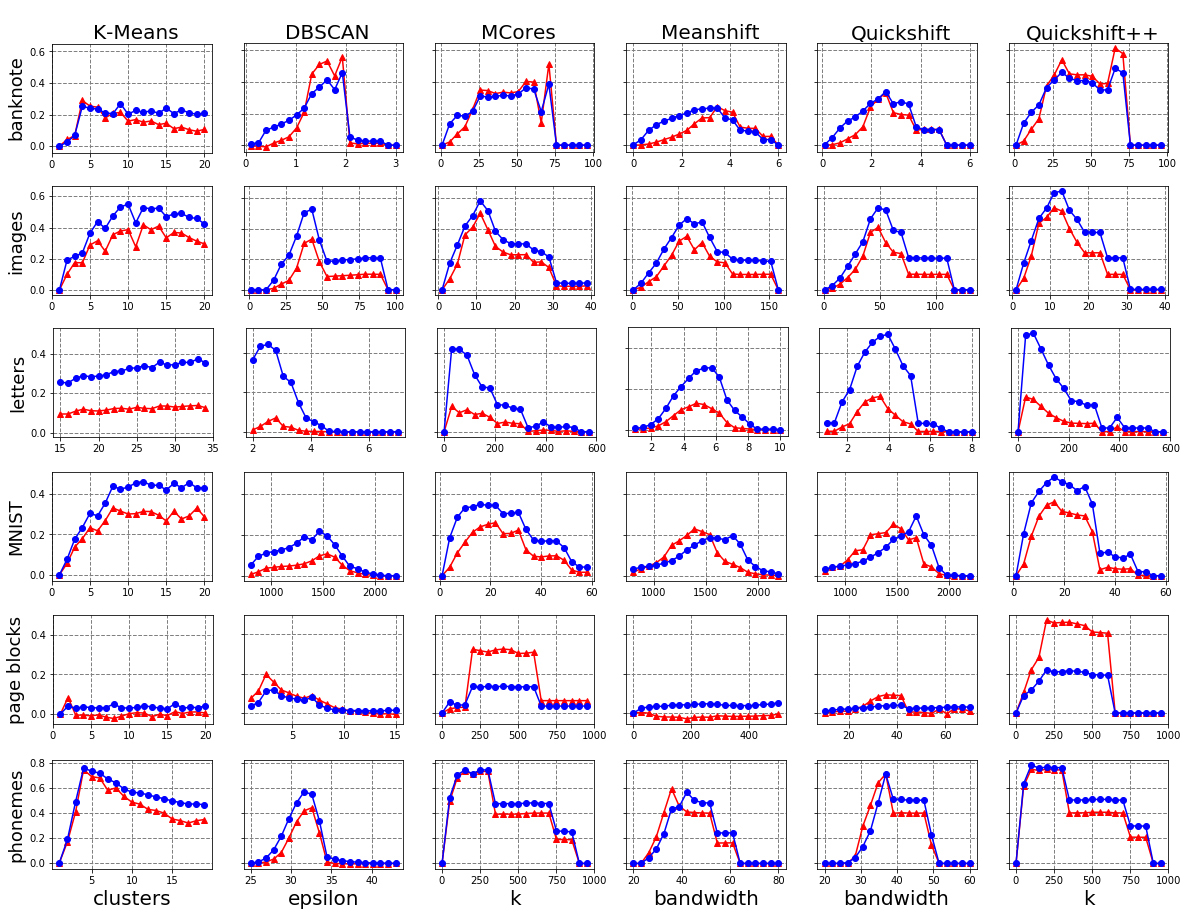}
\end{center}
\caption{\label{figure:stability} For each algorithm, we show clustering performance as a function of its respective hyperparameter setting. The blue line is adj. mutual information and the red line is adj. rand index. Notice that for Quickshift++, we show a wide range of $k$ (relative to $n$), while for the popular procedures, their respective parameters had to be carefully tuned to find the region where the scores are non-trivial.}
\end{figure*}

We ran Quickshift++ against other clustering algorithms on the various real datasets and scored against the ground-truth using the adjusted rand index and the adjusted mutual information scores.

\begin{figure}[H]
\begin{center}
\begin{tabular}{ |p{2.2cm}||p{1cm}|p{0.8cm}|p{1.2cm}| }
 \hline
 Dataset & n & d & Clusters\\
 \hline\hline
 (A) seeds & 210 & 7 & 4 \\
 \hline
 (B) phonemes& 4509 & 258 & 5\\
 \hline
 (C) iris  & 150 & 4 & 3\\
 \hline
  (D) banknote & 1372 & 4 & 2 \\
 \hline 
  (E) images & 210 & 19 & 7\\
 \hline
 (F) letters & 20000 & 16 & 26\\
 \hline
 (G) MNIST  & 1000 & 784 & 10 \\
 \hline
 (H) page blocks & 5473 & 10 & 5\\
 \hline
 (I) glass & 214 & 19 & 7\\
 \hline
\end{tabular}
\caption{\label{fig:datasetsummary}Summary of datasets used, including dataset size ($n$), number of features ($d$) and number of clusters.}
\end{center}
\end{figure}

{\bf Datasets Used}: Summary of the datasets can be found in Figure~\ref{fig:datasetsummary}. Seeds, glass, and iris are standard UCI datasets \cite{Lichman:2013} used for clustering. Banknote is another UCI dataset which involves identifying whether a banknote is forged or not, based on various statistics of an image of the banknote. Page Blocks is a UCI dataset which involves determining the type of a portion of a page (e.g. text, image, etc) based on various statistics of an image of the portion.
Phonemes \cite{friedman2001elements} is a dataset which involves the log periodograms of spoken phonemes. Images is a UCI dataset called Statlog, based on features extracted from various images, and letters is the UCI letter recognition dataset. We also used a small subset of MNIST \cite{lecun2010mnist} for our experiments.

\begin{figure}[H]
\begin{center}
\begin{tabular}{ |p{0.2cm}||p{0.8cm}|p{0.8cm}|p{0.8cm}|p{0.8cm}|p{0.8cm}|p{0.8cm}| }
 \hline
  & KMns & DScn & MCrs & MSft & QSft & QS++\\
 \hline\hline
 A & .7092 &.4473 & .{\bf \color{OliveGreen}7327}  & .{\bf \color{orange}7319} & .6715 & .7261 \\
       & .6738 & .4429 &  .{\bf \color{orange}6872} & .6769 & .6360 & .{\bf \color{OliveGreen}7085} \\
 \hline
  B & .{\bf \color{orange}7432}  & .4458 & .7361 & .5974 & .7165 & .{\bf  \color{OliveGreen}7530} \\
       & .{\bf \color{orange}7574} & .5731 & .7479 & .5700 & .7149 & .{\bf  \color{OliveGreen}7870} \\
 \hline
 C     & .7294 & .5898 & .{\bf \color{orange}7261} & .7028 & .6203 & .{\bf  \color{OliveGreen}7399} \\
          & .{\bf \color{orange}7418} & .5865 & .7265 & .6106 & .5836 & .{\bf \color{OliveGreen}7424} \\
 \hline
 D & .2893 & .{\bf \color{orange}5584} & .5145 & .2434 & .3318 & .{\bf \color{OliveGreen}6152} \\
          & .2690 & .{\bf \color{orange}4594} & .3857 & .2351 & .3397 & .{\bf \color{OliveGreen}4866} \\
 \hline 
 E   & .4177 & .3313 & .{\bf \color{orange}5008} & .3497 & .4077 & .{\bf \color{OliveGreen}5359} \\
          & .5497 & .5264 & .{\bf \color{orange}5814} & .4656 & .5364 & .{\bf \color{OliveGreen}6456} \\
 \hline
 F & .1384 & .0705 & .1284 & .1287 & .{\bf \color{OliveGreen}1793} & .{\bf \color{orange}1766} \\
         & .3741 & .4422 & .4217 & .3027 & .{\bf \color{orange}4940} & .{\bf \color{OliveGreen}5001} \\
 \hline
 G   & .{\bf \color{orange}3320} & .1070 & .2584 & .2281 & .2503 & .{\bf \color{OliveGreen}3606} \\
         & .{\bf \color{orange}4629} & .2164 & .3483 & .1958 & .2911 & .{\bf \color{OliveGreen}4806} \\
 \hline
 H & .0830 & .1962 & .{\bf \color{orange}3251} & .0028 & .0925 & .{\bf \color{OliveGreen}4727} \\
             & .0524 & .1179 & .{\bf \color{orange}1363} & .0526 & .0397 & .{\bf \color{OliveGreen}2192} \\
 \hline
 I & .2770 & .2844 & .2647 & .2790 & .{\bf \color{OliveGreen}2929} & .{\bf \color{orange}2849} \\
             & .3865 & .3542 & .3523 & .3858 & .{\bf \color{orange}4195} & .{\bf \color{OliveGreen}4250} \\
 \hline
\end{tabular}
\caption{\label{figure:results}For each dataset, the first row is the adjusted rand index scores and the second row is the adjusted mutual information scores. Bolded are {\color{OliveGreen} \bf highest} and {\color{orange} \bf second highest} scores.
For MCores and Quickshift++, we used a single $\beta = 0.3$ for each dataset with the exception of for banknote where $\beta = 0.7$. Then the procedures were tuned in their respective essential hyperparameter: $k$-means (KMns) number of clusters, DBSCAN (DScn) epsilon, MCores (MCrs) $k$ from $k$-NN, mean shift (MSft) bandwidth, quick shift (QSft)  bandwidth, Quickshift++ (QS++) $k$.}
\end{center}
\end{figure}

We evaluate performance under the Adjusted Mutual Information and Rand Index scores \citep{vinh2010information} which are metrics to compare clusterings. Not only do we show that Quickshift++ considerably outperforms the popular density-based clustering procedures under optimal tuning (Figure~\ref{figure:results}), but that it is also robust in its hyperparameter $k$ (Figure~\ref{figure:stability}), all while fixing $\beta = 0.3$ for all but one of the datasets. Such robustness to its tuning parameters is highly desirable since optimal tuning is usually not available in practice.

\section{Conclusion}

We presented Quickshift++, a new density-based clustering procedure that first estimates the cluster-cores of the density, which are locally high-density regions. Then remaining points are assigned to its appropriate cluster-core using a hill-climbing procedure based on Quick Shift. 
Such cluster-cores turn out to be more stable and expressive representations of the possibly complex clusters than point-modes. As a result, Quickshift++ enjoys the advantages of the popular density-based clustering algorithms while avoiding many of their respective weaknesses. We then gave guarantees for cluster recovery. Finally, we showed that the algorithm has {\it strong and robust} performance on real datasets and has promising applications to image segmentation.

\newpage

\section*{Acknowledgements}
We thank the anonymous reviewers for their helpful feedback.

\bibliography{paper}
\bibliographystyle{icml2018}


\end{document}